\documentclass[letterpaper]{article} % DO NOT CHANGE THIS
\usepackage[submission]{aaai23}  % DO NOT CHANGE THIS
\usepackage{times}  % DO NOT CHANGE THIS
\usepackage{helvet}  % DO NOT CHANGE THIS
\usepackage{courier}  % DO NOT CHANGE THIS
\usepackage[hyphens]{url}  % DO NOT CHANGE THIS
\usepackage{graphicx} % DO NOT CHANGE THIS
\usepackage{natbib}  % DO NOT CHANGE THIS AND DO NOT ADD ANY OPTIONS TO IT
\usepackage{caption} % DO NOT CHANGE THIS AND DO NOT ADD ANY OPTIONS TO IT
\usepackage{subfigure}
\usepackage{algorithm}
\usepackage{algorithmic}
\usepackage{amsfonts}
\usepackage{amsmath}
\usepackage{amsthm}
\usepackage{newfloat}
\usepackage{listings}

\DeclareCaptionStyle{ruled}{labelfont=normalfont,labelsep=colon,strut=off} % DO NOT CHANGE THIS
\floatstyle{ruled}
\newfloat{listing}{tb}{lst}{}
\floatname{listing}{Listing}

\title{Entropy Augmented Reinforcement Learning}
\author{
    Jianfei Ma
}
\affiliations{
    Northwestern Polytechnical University\\
    School of Mathematics and Statistics\\
    matrixfeeney@gmail.com
}

\newtheorem{theorem}{Theorem}
\newtheorem{proposition}{Proposition}
\newtheorem{lemma}{Lemma}

\newtheorem{definition}{Definition}

\begin{document}
\maketitle

\begin{abstract}
  Deep reinforcement learning was instigated with the presence of trust region methods, being scalable and efficient. However, the pessimism of such algorithms, among which it forces to constrain in a trust region by all means, has been proven to suppress the exploration and harm the performance. Exploratory algorithm such as SAC, while utilizes the entropy to encourage exploration, implicitly optimizing another objective yet. We first observed this inconsistency, and therefore put forward an analogous augmentation technique, which combines well with the on-policy algorithms, when a value critic is involved. Surprisingly, the proposed method consistently satisfies the soft policy improvement theorem, while being more extensible. As the analysis advises, it is crucial to control the temperature coefficient to balance the exploration and exploitation. Empirical tests on MuJoCo benchmark tasks show that the agent is heartened towards higher reward regions, and enjoys a finer performance. Furthermore, we verify the exploration bonus of our method on a set of custom environments.
\end{abstract}

\section{Introduction}

Deep model-free reinforcement learning (RL) algorithms have been applied to a rich array of challenging domains, such as video games \cite{DBLP:journals/nature/MnihKSRVBGRFOPB15} and robot locomotion \cite{DBLP:conf/icml/SchulmanLAJM15}. Policy-based method is one of the most appealing ways to solve those tasks, which can be directly parameterized by neural networks. However, due to high nonlinearity of policy representation and poor sample efficiency, it is hard to obtain favorable performance with standard policy gradient methods. Such dilemma has been posed previously \cite{DBLP:journals/ftrob/DeisenrothNP13}, which had become a major challenge for efficient policy search, robust algorithm design and complex policy learning.

TRPO and PPO are two representative solution methods aiming to overcome those challenges. Both of them successfully improves the sample efficiency and adapts to large-scale problems. Stemming from a same concept, they are yearning to keep the difference between adjacent policies small, so as to squeeze out the available data at hand. To arrive at that point, TRPO \cite{DBLP:conf/icml/SchulmanLAJM15} performs on a constraint optimization problem, by enforcing the KL divergence, while its successor -- PPO \cite{DBLP:journals/corr/SchulmanWDRK17} simplifies the procedure by clipping importance sampling ratio to exclude all suspicious gradients. However, the sacrifice would be lack of exploration when imposed with hard constraints. Determining the radius of the trust region is arduous, while it may evolve across learning. And \cite{DBLP:journals/corr/abs-2009-10897} found that PPO tends to follow suboptimal actions. This urges a refinement that encourages the agent to explore more while sustain the certain constraint for sample reuse.

A natural approach would be reward shaping with an augmenting entropy. This idea was studied concurrently in \cite{DBLP:conf/nips/NachumNXS17} and \cite{DBLP:conf/icml/HaarnojaTAL17}, recently advanced by \cite{DBLP:conf/icml/HaarnojaZAL18} with an actor-critic algorithm -- SAC. It is exploratory by grace of the additional entropy bonus, and sample-efficient with off-policy data fully engaged. However, there is an inconsistency between the augmented objective and the evaluating function (specifically, the action-value function) that actually follows to optimize. This unawareness carries even on the newer version \cite{DBLP:journals/corr/abs-1812-05905} with automated entropy adjustment. We argue that SAC does not follow the objective proposed whereas optimizes another objective. We therefore propose an alternative Bellman operator that is conjugate with that of SAC, which satisfies the soft policy improvement theorem as well, without modification. In addition, it combines well with the on-policy algorithms, compatible with the variance reduction mechanism GAE \cite{DBLP:journals/corr/SchulmanMLJA15}, while its counterpart fails. 

In this paper, we come up with a general entropy augmentation method, which boosts the exploration capability of the agent for both on- and off-policy algorithms. We first point out the inconsistency of SAC. Then we formally deliver our method, with analysis and comparison provided. Lastly, we present empirical results that show our method can substantially improve the performance and encourage the exploration.
  
\section{Preliminaries}
We consider an infinite-horizon discounted MDP, which models how the agent interacts with the environment. It can be defined by a tuple $\mathcal{M} = (\mathcal{S}, \mathcal{A}, P, r, \rho_{0}, \gamma)$, in which $\mathcal{S}$ is the state space, $\mathcal{A}$ is the action space, $p: \mathcal{S} \times \mathcal{A} \times \mathcal{S} \rightarrow \mathbb{R}$ is the transition probability distribution, $r: \mathcal{S} \times \mathcal{A} \rightarrow \mathbb{R}$ is the reward function, $\rho_{0}: \mathcal{S} \rightarrow \mathbb{R}$ is the distribution of the initial state $s_{0}$, and $\gamma \in [0, 1)$ is the discount factor. We denote $\pi : \mathcal{S} \times \mathcal{A} \rightarrow [0, 1]$ as a stochastic policy, usually parameterized as $\pi_{\theta}$. As it evolves, we can collect a trajectory $\tau = (s_{0}, a_{0}, s_{1}, a_{1}, \dots)$, upon which the agent yearns to maximize the expected discounted reward $\eta (\pi_{\theta}) = \mathbb{E}_{\tau}\left[\sum\limits_{t=0}^{\infty}\gamma^{t}r_{t}  \right]$. We also define the unnormalized discounted state visitation distribution (with abuse of notation) as $\rho_{\pi} = \sum\limits_{t=0}^{\infty} \gamma^{t} P(s_{t} = s | \rho_{0}, \pi)$. Sometimes we would abbreviate $\pi_{\theta}$ as $\pi$ for simplicity's sake. Throughout the paper, we denote the entropy of $\pi(\cdot|s)$ as $\mathcal{H}(s)$.

\subsection{Performance Difference Lemma}
The performance-difference lemma \cite{DBLP:journals/mor/BurnetasK97} \cite{DBLP:conf/icml/KakadeL02} \cite{DBLP:conf/icml/SchulmanLAJM15} states that given two policies $\tilde{\pi}$ and $\pi$
\begin{equation}
  \label{eq:14}
  \eta (\tilde{\pi}) = \eta (\pi) + \mathbb{E}_{s \sim \rho_{\tilde{\pi}}, a \sim \tilde{\pi}} [A^{\pi}(s, a)]
\end{equation}
where $A^{\pi}(s, a) = Q^{\pi}(s, a) - V^{\pi}(s)$ is the advantage function. This is useful since it bridges any two policies by the advantage of one side.

It matches to the first order as $\tilde{\pi} = \pi$, therefore generalizes standard policy gradient theorem \cite{DBLP:conf/nips/SuttonMSM99}. And it is echoed with trust region methods as well as long as one replaces $\rho_{\tilde{\pi}}$ with $\rho_{\pi}$, irrespective of the state distribution mismatch, with the compensation of keep the dissimilarity between the adjacent policies small.

\subsection{Generalized Advantage Estimation}
\label{sec:gener-advant-estim}
GAE \cite{DBLP:journals/corr/SchulmanMLJA15} is an effective variance reduction scheme for policy gradient methods, which telescopes the n-step rolling advantage $\hat{A}_{t}^{n} = \sum\limits_{l=0}^{n-1} \gamma^{l}r(s_{t + l}, a_{t + l}) + \gamma^{n}V(s_{t + n}) - V(s_{t})$ by exponential average
\begin{equation}
  \label{eq:15}
  \begin{aligned}
  \hat{A}_{t}^{\text{GAE}(\gamma, \lambda)} & = (1 - \lambda)\sum\limits_{n=1}^{\infty}\lambda^{n}\hat{A}_{t}^{n}\\
  & = \sum\limits_{l = 0}^{\infty}(\gamma \lambda)^{l} \delta_{t + l}    
  \end{aligned}
\end{equation}
where $\delta_{t} = r(s_{t}, a_{t}) + \gamma V(s_{t + 1}) - V(s_{t})$. Controlling this additional parameter $\lambda$ properly can substantially reduce the variance of policy gradient estimate.

\subsection{Soft Bellman Operator}
In SAC \cite{DBLP:conf/icml/HaarnojaZAL18}, it defines a modified Bellman backup operator $\mathcal{T}^{\pi}$
\begin{equation}
  \label{eq:16}
  \mathcal{T}^{\pi} Q(s_{t}, a_{t}) = r(s_{t}, a_{t}) + \gamma \mathbb{E}_{s_{t + 1}}[V(s_{t + 1})],
\end{equation}
where
\begin{equation}
  \label{eq:17}
  V(s_{t}) = \mathbb{E}_{a_{t} \sim \pi}[Q(s_{t}, a_{t}) - \log{\pi(a_{t} | s_{t})}]
\end{equation}
With repeatedly applying this operator over the action-value function $Q$, it approaches to the optimal soft value function, and optimal policy is induced consequently.

Based on this appraisal, it proposes to minimize the following projection loss for any $s \in \mathcal{S}$ over the softmax policy class $\Pi$
\begin{equation}
  \label{eq:prj}
  \mathcal{L}(\pi' | s) = D_{\text{KL}}\Biggl(\pi'(\cdot | s) \| \frac{\exp{(\frac{1}{\alpha}Q^{\pi}(s, \cdot))}}{Z^{\pi} (s)}\Biggr)
\end{equation}

On top of that, it is stated that it optimizes the maximum entropy objective
\begin{equation}
  \label{eq:18}
  J(\pi) = \mathbb{E}_{\tau}\left[\sum\limits_{t=0}^{\infty}\gamma^{t}\bigl(r(s_{t}, a_{t}) + \alpha \mathcal{H}(s_{t})\bigr)  \right]
\end{equation}

While this elaborate connection is seemingly consistent, it would be wrong that it doesn't optimize this objective but follow another one implicitly, as later that will be discussed.

\section{Optimization Inconsistency }
\label{sec:optim-incons-}
\begin{definition}
  The reward shaping is absorbing if there exists a real-valued function $f: \mathcal{S} \times \mathcal{A} \mapsto \mathbb{R}$ such that for $\forall t$, the reward function can be reformulated as
  \begin{equation}
    \label{eq:19}
    \hat{r}(s_{t}, a_{t}) = r(s_{t}, a_{t}) + f(s_{t}, a_{t})
  \end{equation}
\end{definition}

With this definition, we have the following statement

\begin{proposition}
  Soft Bellman backup operator is not compatible with GAE \ref{sec:gener-advant-estim}.
\end{proposition}
\begin{proof}
  It can be noted that $\mathcal{T}^{\pi}Q$ is not absorbing w.r.t. $V$, by taking a look at 1-step expansion
  \begin{equation*}
    \label{eq:20}
    \mathcal{T}^{\pi}Q_{t} = r_{t} + \gamma \mathbb{E}_{p, \pi}[r_{t + 1} + \gamma \mathbb{E}_{p} [V_{t + 2}] - \alpha \log{\pi_{t + 1}}]
  \end{equation*}
  The augmenting term $- \alpha \log{\pi_{t + 1}}$ can only associate with either $r_{t}$ or $r_{t + 1}$, therefore it breaks the tie, consequently not absorbing.
  
  Immediately it fails to derive Equation \ref{eq:15}, so that the statement follows.
\end{proof}

This disconnection further indicates that
\begin{proposition}
  SAC does not optimize the soft objective w.r.t. Equation \ref{eq:18}.
  \begin{proof}
    By unrolling the soft Bellman backup operator when the true value $Q^{\pi}$ is attained
    \begin{equation*}
      \mathcal{T}^{\pi}Q^{\pi}_{t}  = \mathbb{E}_{s_{t+1}, a_{t+1}, \dots}\Biggl[\sum\limits_{l=1}^{\infty} \gamma^{l}\bigl(r_{t + l} +  \alpha \mathcal{H}_{t + l} \bigr) + r_{t} \Biggr]
    \end{equation*}
    it is clear that it is not absorbing at the first time step, whereas the soft objective is consistently absorbing for all time steps, thus it contradicts.
  \end{proof}
\end{proposition}

The incompatibility with GAE would limit its use for the on-policy algorithms, especially when a value critic involved only. In next part, we will deliver our method formally, and answer that why even this inconsistency exists, SAC still works.

\section{Entropy Augmented Reinforcement Learning}
We consider an entropy augmented expected discounted reward function, with notation overloaded
\begin{equation}
  \label{eq:6}  
    \eta (\pi) = \mathbb{E}_{\tau}\Biggl[\sum\limits_{t=0}^{\infty} \gamma^{t}\bigl(r(s_{t}, a_{t}) + \gamma \alpha \mathbb{E}_{s_{t+1}} \mathcal{H} (s_{t+1})\bigr) \Biggr]
\end{equation}

The temperature $\alpha$ determines the relative importance of the expected entropy of next state $s_{t + 1}$ against the reward. And thus it gives the agent a hindsight to choose an action that has averagely high uncertainty.

We define a ``bootstrap'' Bellman operator analogous to the soft Bellman operator
\begin{equation}
  \label{eq:21}
  \mathcal{T}^{\pi} Q(s_{t}, a_{t}) = r(s_{t}, a_{t}) + \gamma \mathbb{E}_{s_{t + 1}}[V(s_{t + 1}) + \alpha \mathcal{H}(s_{t + 1})],
\end{equation}
where
\begin{equation}
  \label{eq:22}
  V(s_{t}) = \mathbb{E}_{a_{t} \sim \pi}[Q(s_{t}, a_{t})]
\end{equation}

By extracting the entropy out, and associate with the reward, it is not difficult to see that $\eta(\pi) = \mathbb{E}_{s_{0}}[V^{\pi}(s_{0})] = \mathbb{E}_{s_{0}, a_{0}}[Q^{\pi}(s_{0}, a_{0})]$, thus it is consistent in the optimization scenario.

This association enables it to be combined with the variance reduction method such as GAE, by simply transforming reward function as $\hat{r}(s, a) = r(s, a) + \gamma \alpha \mathcal{H}(s')$, when simulation involved.

Similarly, we can derive a performance-difference lemma for the augmented objective
\begin{equation}
  \label{eq:22}
  \eta (\tilde{\pi}) = \eta (\pi) + \mathbb{E}_{\substack{s \sim \rho_{\tilde{\pi}}\\ a \sim \tilde{\pi}\\ s' \sim p}} [A^{\pi}(s, a) + \gamma \alpha (\tilde{\mathcal{H}}(s') - \mathcal{H}(s'))]
\end{equation}
With an old policy $\pi$ and a parameterized policy $\pi_{\theta}$ given, we denote that
\begin{equation}
  \label{eq:obj}
  \mathcal{L}_{\pi}(\pi_{\theta}) = \mathbb{E}_{\rho_{\pi_{\theta}}, \pi_{\theta}, p} [A^{\pi}(s, a) + \gamma \alpha \mathcal{H}^{\pi_{\theta}}(s')]
\end{equation}
And since $\mathcal{H}^{\pi}$ is state-dependent only, it is not difficult to derive that $\nabla_{\theta}\eta(\pi_{\theta}) = \nabla_{\theta}\mathcal{L}_{\pi}(\pi_{\theta})$.

For trust region methods, whenever the successor policy $\pi_{\theta}$ is as close as the origin policy $\pi$, it instead approximately optimizes
\begin{equation}
  \label{eq:approxobj}
  \hat{\mathcal{L}}_{\pi}(\pi_{\theta}) = \mathbb{E}_{\rho_{\pi}, \pi_{\theta}, p} [A^{\pi}(s, a) + \gamma \alpha \mathcal{H}^{\pi_{\theta}}(s')]
\end{equation}

While it is appealing to incorporate the entropy bonus to encourage exploration, we don't know how well this approximation is and how the temperature parameter affects the learning.

\subsection{Theoretical Analysis}
We firstly answer the first question by stating that
\begin{theorem}
  \label{thm:1}
  Let $\zeta = D_{\text{KL}}^{\text{max}}(\pi_{\text{old}}, \pi_{\text{new}})$. Then the following bound holds:
  \begin{equation}
    \label{eq:10}
    \begin{array}{c}
      \eta (\pi_{\text{new}}) \geq \hat{\mathcal{L}}_{\pi_{\text{old}}}(\pi_{\text{new}}) - \frac{4 \epsilon \gamma}{(1 - \gamma)^{2}} \zeta^{2} \\
      \text{where}\ \epsilon = \max_{s,a}|A^{\pi}(s, a)|
    \end{array}
  \end{equation}
\end{theorem}
This means as long as the dissimilarity between adjacent policies is small under the metric KL-divergence, the approximation is considerably good, which is reminiscent of TRPO, allowing the data at hand to be fully utilized while encouraging the exploration.

We also define the Bellman backups w.r.t. $V$ for later use
\begin{equation}
  \label{eq:11}
  \begin{aligned}
    \mathcal{T}_{\alpha}^{\pi}V(s) & = \mathbb{E}_{a \sim \pi}\bigl[\hat{r}(s, a) + \gamma \mathbb{E}_{s' \sim p}[V (s')]\bigr] \\
    \mathcal{T}_{\alpha}^{*}V(s) & = \max_{\pi}{\mathbb{E}_{a \sim \pi}\bigl[\hat{r}(s, a) + \gamma \mathbb{E}_{s' \sim p}[V (s')]\bigr]}
  \end{aligned}  
\end{equation}
where $\hat{r}(s, a) = r(s, a) + \gamma \alpha \mathbb{E}_{s' \sim p}[\mathcal{H}(s')]$, and subscript indicates the relation with $\alpha$.

For which the following holds:
\begin{lemma}
  \label{lm:2}
  Assume $|r| < C_{r}$ and $|\mathcal{H}| < C_{\mathcal{H}}$, then both $\mathcal{T}_{\alpha}^{\pi}$ and $\mathcal{T}_{\alpha}^{*}$ are $\gamma$-contraction.
\end{lemma}
This is useful since it tells us that the optimal value function is unique.

In retrospect of the reward shaping, while any augmentation to the reward will result in a shifted MDP $M'$ against the original MDP $M$, \cite{DBLP:conf/icml/NgHR99} first proved that there is only a type of augmentation -- potential-based shaping function $f(s, a, s') = \gamma \varPhi(s') - \varPhi(s)$, where $\varPhi: \mathcal{S} \rightarrow \mathbb{R}$ is a real-valued function that sustains the policy invariance, that is, optimal policies $\pi_{M'}^{*}$ in $M'$ and $\pi_{M}^{*}$ in $M$ are identical. Unfortunately, our augmenting term is not a potential-based shaping function, thus it may not sustain the policy invariance. But the following theorem tells us that if we appropriately control the temperature $\alpha$, we can asymptotically sustain the policy invariance, on the basis of Lemma \ref{lm:2}:

\begin{theorem}
  \label{thm:3}
  \textbf{(Optimal Policy Error Bound)} the error between the augmented value function $\tilde{V}^{*}$ and the original counterpart $V^{*}$ can be bounded as:
  \begin{equation}
    \label{eq:12}
    \|\tilde{V}^{*} - V^{*}\|_{\infty} \leq \frac{\gamma}{1 - \gamma} \alpha C_{\mathcal{H}}
  \end{equation}
\end{theorem}

This theorem says that so long as we let $\alpha \rightarrow 0$, then we are guaranteed any $\pi_{M'}^{*}$ (where $M'$ is the MDP shifted by our method) we might be trying to learn will also be optimal in $M$ (the original MDP). It hints us that annealing the temperature $\alpha$ is crucial while encouraging the agent to explore.

And surprisingly, the method also satisfies the soft policy improvement theorem stated in SAC, without any further modification. 
\begin{theorem}
  \label{thm:soft}
  Let $\tilde{\pi}$ be within $\Pi$, minimizing the projection loss \ref{eq:prj} against $\pi$ for any $s \in \mathcal{S}$. Then $Q^{\tilde{\pi}}(s_{t}, a_{t}) \geq Q^{\pi}(s_{t}, a_{t})$ for all $(s_{t}, a_{t}) \in \mathcal{S} \times \mathcal{A}$.
\end{theorem}

Combining the following statement
\begin{proposition}
  Bootstrap Bellman operator is conjugate with soft Bellman operator w.r.t. $Q$
  \begin{proof}
    By plugging Equation \ref{eq:22} into Equation \ref{eq:21} and rewriting the entropy as the expectation of the log probability, then it follows. 
  \end{proof}
\end{proposition}
it reveals that two operators share with one $Q^{\pi}$ at the fixed point, which indicates that SAC is optimizing Equation \ref{eq:6}.

This explains why the newer version of SAC \cite{DBLP:journals/corr/abs-1812-05905} can drop the value critic, and still guarantee performance improvement. Even when the value critic appeared in \cite{DBLP:conf/icml/HaarnojaZAL18}, it still goes along with $Q$ function, where the value function just serves for a better regression.

\subsection{Practical Algorithm}
\begin{algorithm}[tb]
\caption{EARL for On-Policy Algorithms}
\label{alg:EARL-on}
\textbf{Input}: temperature $\alpha$\\
\textbf{Initial Parameter}: $\theta_{0}$
\begin{algorithmic}[1] %[1] enables line numbers
  \FOR{iteration $i \gets 0, 1, \dots, N - 1$}{
    \FOR{step $t \gets 0, 1, \dots, M - 1$}{
      \STATE Execute policy $\pi_{\theta}$ in the environment
      \STATE Compute $\hat{r}(s_{t}, a_{t}) = r(s_{t}, a_{t}) + \gamma \alpha \mathcal{H}(s_{t+1})$
      \STATE Store experience in the batch $\mathcal{B}$
    }    
    \ENDFOR
    \STATE Approximate advantage function $\hat{A}$ for the batch $\mathcal{B}$
    \FOR{update $e \gets 0, 1, \dots, L - 1$}{
      \STATE Sample a mini-batch of $A_{t}$, and augment as $A_{t} + \gamma \alpha \mathcal{H}^{\pi_{\theta}}(s_{t + 1})$
      \STATE Update policy and value function by on-policy algorithms
    }
    \ENDFOR
  }
  \ENDFOR
\end{algorithmic}
\end{algorithm}

\begin{algorithm}[tb]
\caption{EARL Soft Actor-Critic}
\label{alg:EARL-sac}
\textbf{Input}: temperature $\alpha, \tau$\\
\textbf{Initial Parameter}: $\bar{\phi}, \phi, w, \theta$
\begin{algorithmic}[1] %[1] enables line numbers
    \FOR{step $t \gets 0, 1, \dots, M - 1$}{
        \STATE Execute policy $\pi_{\theta}$ in the environment
        \STATE Store transition $(s_{t}, a_{t}, r_{t}, s_{t + 1})$ to the replay buffer $\mathcal{D}$
        \STATE Sample mini-batch of $n$ transitions $(s, a, r, s')$ from $\mathcal{D}$
        \STATE Compute $y_{Q} = r(s, a) + \gamma (V_{\bar{\phi}}(s') + \alpha \mathcal{H}(s'))$
        \STATE Update critic $w = w - \lambda_{w} \nabla_{w}{\frac{1}{n}\sum(Q_{w} - y_{Q})^{2}}$
        \STATE Sample action $a'$ for $s$, update actor w.r.t. the projection loss \ref{eq:6}
        \STATE Compute $y_{V} = Q_{w}(s, a')$
        \STATE Update value net $\phi \gets \phi - \lambda_{\phi} \nabla_{\phi}{\frac{1}{n}\sum(V_{\phi} - y_{V})^{2}}$
        \STATE Update target net $\bar{\phi} \gets (1 - \tau) \bar{\phi} + \tau \phi$
    }    
    \ENDFOR
\end{algorithmic}
\end{algorithm}
In practice, we would have no knowledge about the environment dynamics, therefore need resort to approximation for the surrogate function. For on-policy domain, there are only two places that we need care about, the reward transition and advantage augmentation as informed by Equation \ref{eq:obj}. As the data is sequentially correlated, it is expected that $r_{t} + \gamma \alpha \mathcal{H}_{t + 1}$ and $A_{t} + \gamma \alpha \mathcal{H}^{\pi_{\theta}}_{t + 1}$ are unbiased estimators of $\hat{r}$ and $\mathcal{L}_{\pi}(\pi_{\theta})$ respectively. And for a variant of SAC, the difference only lies at the target calculation for $Q$ and $V$. If the entropy has no closed form, one can just sample a few of samples to approximate. Full algorithms are listed at Algorithm \ref{alg:EARL-on} and \ref{alg:EARL-sac} accordingly.

The approach unifies different domains, which is general, simple and can be an effective tool to boost the exploration capability of the agent.

\begin{figure*}[ht]
\centering
\includegraphics[width=1\textwidth]{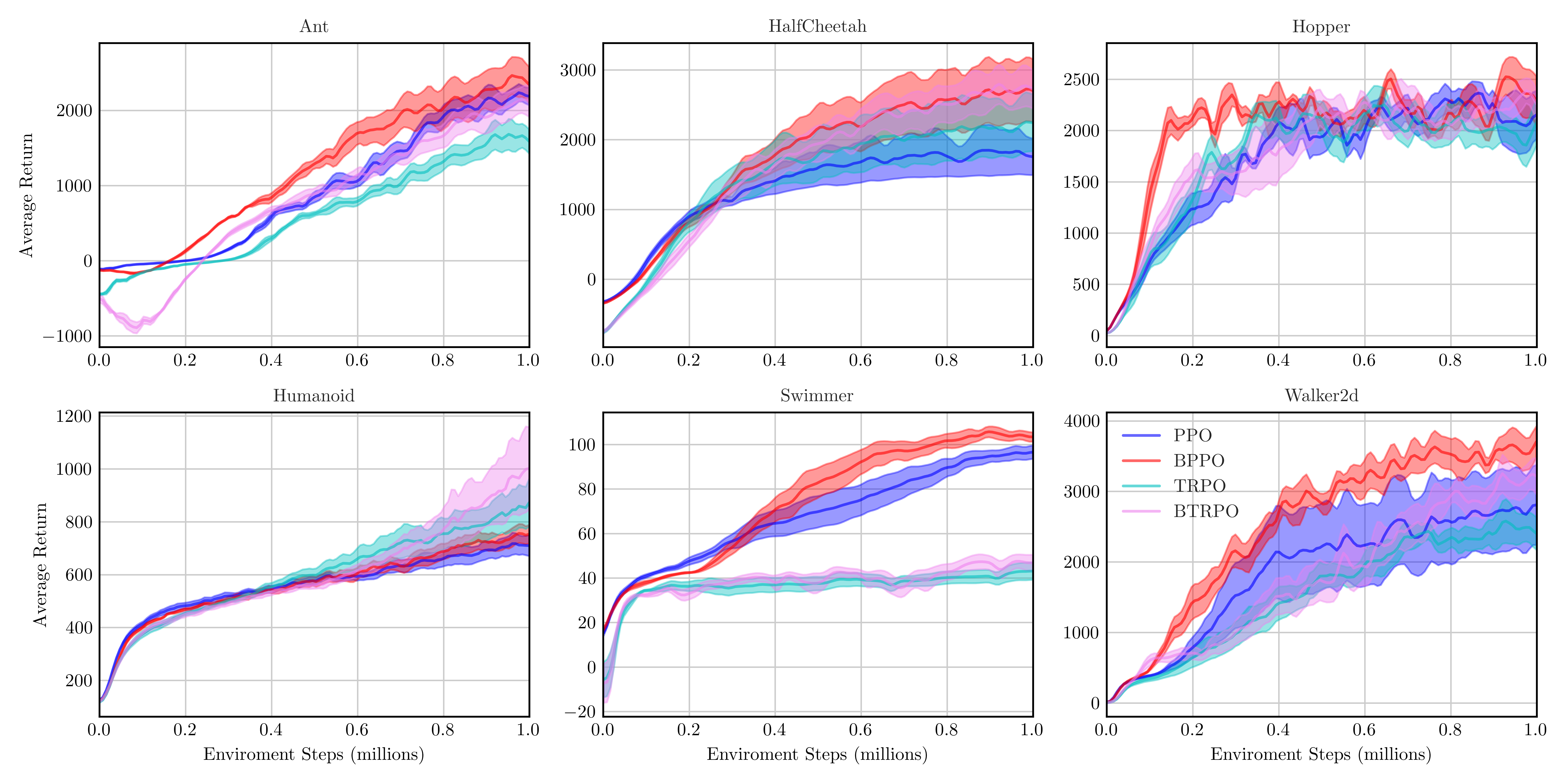}
\caption{Training curves on a set of MuJoCo tasks performed by the augmented TRPO and PPO against their baselines, averaged over 5 seeds and shaded with standard error.}
\label{fig:1}
\end{figure*}

\section{Experiments}
In this section, we will test our method on a set of on-policy algorithms such as TRPO and PPO, to answer whether it can promote the exploration capability of those two algorithms. And we make comparison with the first version of SAC \cite{DBLP:conf/icml/HaarnojaZAL18}, when a value critic is involved, to verify the effectiveness of the proposed approach. We conjecture that our method enjoys a better optimization property than SAC due the consistency as mentioned before. We address the importance of controlling the temperature $\alpha$, how it influence the learning process, and how to choose an appropriate annealing schedule. And we also distill the exploration bonus on two grid world environments, to show how the augmentation takes effect on simple environments.

\subsection{Continuous Benchmarks}
We compare our augmentation techniques as illustrated in Algorithm \ref{alg:EARL-on} to the standard TRPO and PPO, across several continuous control tasks from the OpenAI Gym benchmark suite \cite{DBLP:journals/corr/BrockmanCPSSTZ16} with the MuJoCo simulator \cite{DBLP:conf/iros/TodorovET12}. From the Figure \ref{fig:1}, we can see that our method can outperform on all tasks against both TRPO and PPO. We argue that the experimental discovery proposed by \cite{DBLP:conf/icml/AhmedR0S19} that the entropy has an insignificant effect on HalfCheetah environment, whereas our result shows that with the aforementioned techniques, we can achieve evident improvements.

We also compare our method with SAC-v1, which only differs from how it augments the reward. We align our algorithm with SAC-v1, including hyperparameters and policy parameterization. Since we also make use of the squashed Gaussian distribution, which doesn't have a closed-form entropy, therefore we need additional samples to estimate the entropy, by averaging over the log probabilities. Typically, only a few of samples $N = 5$ are enough to learn reliable features and obtain satisfactory performance, and we found it is not sensitive to this choice. The result displayed in Figure \ref{fig:2} shows that our method can outperform SAC-v1 across a set of tasks, being more stable and learning faster. We conjecture that this performance gain is due to the consistency between value functions and the objective to be optimized, in which the regression target is consistent, and therefore enjoying a easier optimization as learning evolves.

\begin{figure*}[t]
\centering
\includegraphics[width=1\textwidth]{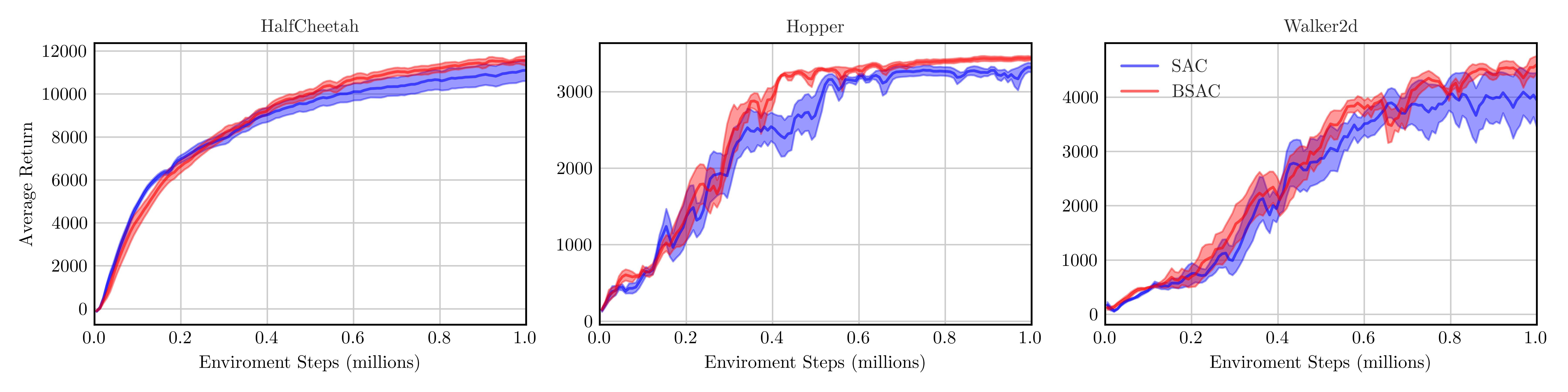}
\caption{Training curves of SAC and its variant on a set of MujoCo tasks, averaged over 5 seeds and shaded with standard error.}
\label{fig:2}
\end{figure*}

\subsection{Annealing Schedule}
While this injected exploration bonus could help the agent explore more, we are not interested quantifying the augmented value function, but in solving the original MDP $M$. As the Theorem \ref{thm:3} suggested, in order to recover to the original MDP, we need to anneal the temperature coefficient. But we also seek to use the augmentation to boost the exploration of the agent, which means we can not anneal it too fast. Inspired from the trick -- exponential decay, to shrink the learning rate in the deep learning, we propose to exponentially drop the temperature every update to control the behavior of the temperature while maintain a sufficient amount of exploration. It requires an additional adjustable parameter $\sigma$ -- the decay rate, that close to $1$ can help the agent to exploit what have learned. To further investigate the impact of different combinations, we consider to explore the temperature's magnitude and whether to decay or not, to demonstrate the effect of large or small temperature and examine the necessity of decaying. The experiments can be grouped by the constant (no decay) and the dynamic (with decay), where their magnitude varies internally. And those are performed on the HalfCheetah and Swimmer environments. 

\begin{figure}[ht]
\centering
\includegraphics[width=\columnwidth]{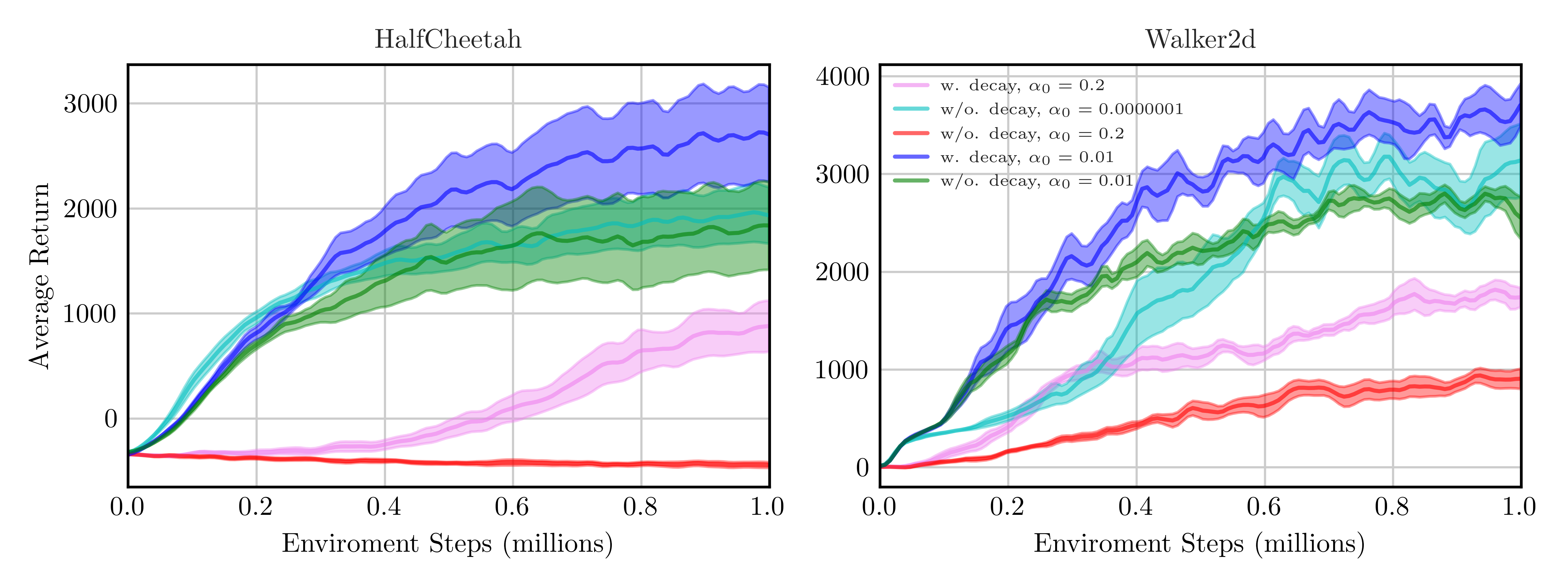}
\caption{Performance of different combinations on HalfCheetah and Walker2d environments averaged over 5 seeds. All augmentations are taken on PPO. The decay rate is consistently equal to $0.99$.}
\label{fig:3}
\end{figure}

The results of the Figure \ref{fig:3} reveal that:
\begin{itemize}
\item \textbf{Large} initial temperature is not preferred than \textbf{small} initial temperature.
\item \textbf{With} decay is more preferred than \textbf{no} decay.
\item \textbf{Small constant} can perform considerably well.
\end{itemize}

The reason for that large initial temperature would harm the performance is that the learning signal is overwhelmed by the augmenting term, which causes the agent to be too exploratory to exploit the beneficial experience. And smaller one makes the reward clearer, that could be more easily taken advantage of by the agent. This is known as the trade-off between exploration and exploitation. However, our results show that if the temperature parameter is properly controlled, we can actually be benefited from it. For practical concern, we suggest that decaying is a good pick in order to generalize across different tasks.

\subsection{Grid Worlds}

In this part, we introduce two grid world environments, one of which is originally from \cite{DBLP:conf/icml/AhmedR0S19}, and the other is a custom environment. The former one is a simple episodic task which has one suboptimum with a reward scalar 4.5 at the right top corner, and one optimum with 5.0 at the left bottom corner. The agent always starts from the left top corner. It is designed to test the discriminative ability of the agent when a suboptimum diverts the attention. The custom environment is more complicated, which has one fixed suboptimum and one random optimum. The suboptimum is at the right top corner with a reward scalar 0.5, if it is seized, the agent will be reset to the start at the left bottom corner, while the optimum with a reward scalar 1.0 is randomly moved to the unoccupied location every time been captured by the agent. We represent the agent, suboptimum, and optimum of the environment as red, green and blue block. We denote two environments as \textbf{Diagonal} and \textbf{TwoColors}, see Figure \ref{fig:4}. And we use $10 \times 10$ size for both environments.

\begin{figure}
\centering
\subfigure[Diagonal environment]{\label{fig:a}\includegraphics[width=40mm]{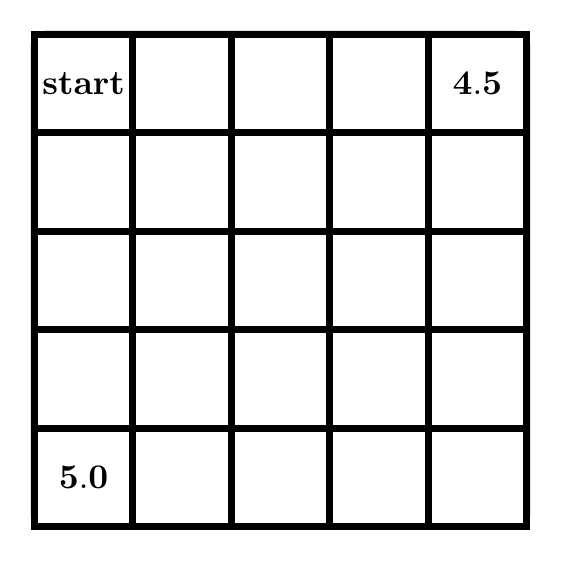}}
\subfigure[TwoColors environment]{\label{fig:b}\includegraphics[width=40mm]{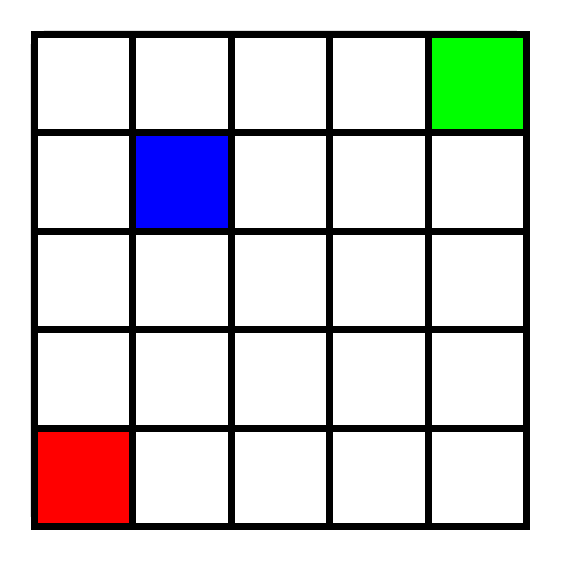}}
\caption{Custom environments}
\label{fig:4}
\end{figure}

Diagonal is about to investigate the effect of the exploration bonus at the previous chapter of the learning process, and to see whether it can speed up the agent to find the optimum. And TwoColors is proposed to investigate whether the augmentation can help the agent to escape from the suboptimums when the goal is randomized.

\begin{figure}[ht]
\centering
\includegraphics[width=\columnwidth]{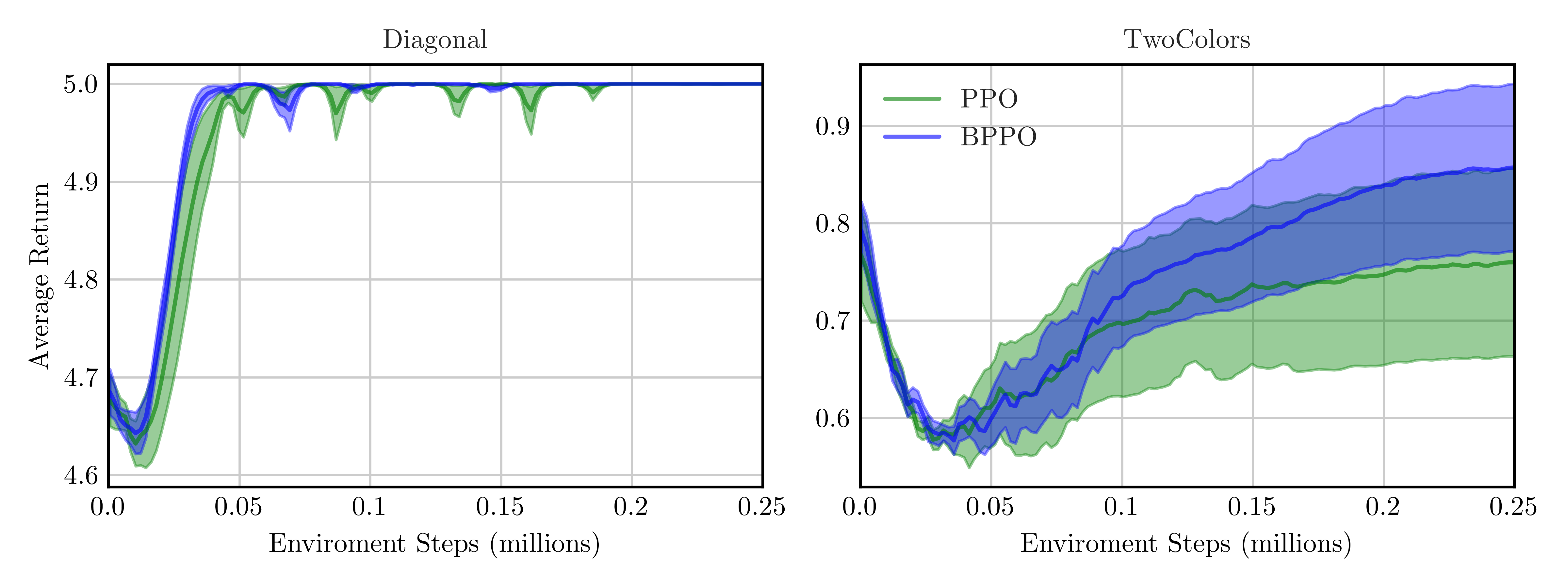}
\caption{Training curves on Diagonal and TwoColors averaged over 5 seeds.}
\label{fig:5}
\end{figure}

As we can see from the Figure \ref{fig:5}, on Diagonal, our method reaches optimum faster, and is more stable after convergence. This verifies that the exploration bonus injected by our method can help the agent explore more, converge fast, and be keener towards the optimum. On the other side, since TwoColors environment is more hard to explore, thus it needs more effort to prevent being trapped at local minimums. However, the result shows that PPO has around $50 \%$ probability averagely to stuck at the suboptimum, while entropy augmented PPO escapes from the suboptimum and reaches a higher reward. Both experiments indicate that our method can boost the exploration of the agent, as a result, achieve performance quicker and higher.

\section{Related Work}
Exploration capability of the agent can sometimes be a bottleneck for RL algorithms. The most intuitive way would be modifying the reward function with some stimulus reflecting the uncertainty of the perception of the environment of the agent. For instance, reward augmentation will change the behavior of the agent by considering the relative importance of the regularization term. Different regularization can take quite different effects, such as conservative with Kullback–Leibler divergence, and exploratory with the entropy. A appropriate choice can be of vital importance. In another possible perspective, the motivation can be extrinsic. \cite{DBLP:conf/iclr/Flet-BerliacFPP21} introduces an adversary that mimics the actor distribution, where the larger discrepancy would help the agent to follow a different strategy. The same idea is also reflected in \cite{DBLP:conf/iclr/BurdaESK19}, motivates the agent to explore the area where the prediction error against a fixed randomly initialized network is large.

Entropy regularization is a natural way to encourage the exploration, since it measures the uncertainty of the policy. Understanding the impact of the entropy was first studied by \cite{doi:10.1080/09540099108946587}, who proposed to incorporate entropy into networks to promote the exploration. And in practice, it is straightforward to utilize it as an add-on, such as in A3C \cite{DBLP:conf/icml/MnihBMGLHSK16}, which is used for regularizing the objective function to discourage the premature convergence. However, \cite{DBLP:journals/corr/NeuJG17} argue that it may not converge. By solving a constraint optimization problem \cite{DBLP:phd/us/Ziebart18}, it can be shown that the optimal policy is the softmax policy w.r.t. some defined preference functions, such that the optimal soft value functions defined in \cite{DBLP:conf/icml/HaarnojaTAL17}. In contrast, its actor-critic successor \cite{DBLP:conf/icml/HaarnojaZAL18} doesn't rely on those equations instead directly project the parameterized policy onto the softmax policy associated with the discounted value functions. By deriving the same fashion, \cite{DBLP:conf/nips/NachumNXS17} connected value and policy based reinforcement learning with the multi-step consistency. In contrast, since the algorithm is based on the value function rather than the action-value function, the aforementioned inconsistency issue doesn't exist. DAC \cite{DBLP:conf/icml/HanS21} imitates the soft Bellman operator while incorporates another so-called sample-aware entropy. It means it also implicitly optimizes another objective.

\cite{DBLP:journals/corr/SchulmanAC17} further investigate the equivalence between policy gradients and soft Q-learning, however, using relative entropy (KL divergence). And \cite{DBLP:conf/nips/YangLZ19} generalize it on several augmentation forms, such as Tsallis entropy etc. However, comparing to our method, those approaches are not absorbing, which limits their use far and wide. Unlike those, our method is general for both on and off-policy algorithms, and is handy to implement, modify and extend on a wide variety of algorithms.

% \cite{DBLP:conf/ijcai/MeiXHS019} adds the entropy bonus to the policy mirror descent method, and induces a similar optimize problem to that of \cite{DBLP:journals/corr/abs-1812-05905}.

% Different types of regularization are unified by \cite{DBLP:journals/corr/NeuJG17}, of which TRPO is viewed as a special case of the mirror descent

% 2. Entropy augmentation and other regularization (relative entropy)
% - MRL \cite{DBLP:conf/nips/VieillardPG20} augments with the log probability that favors the higher-probability action.

% 3. Improvement over SAC

% 4. What entropy can do else?
% - more robust \cite{DBLP:conf/iclr/EysenbachL22}, succeed in maximizing the worst-case reward
% - entropy regularization may speed up convergence \cite{DBLP:conf/icml/MeiXSS20}

% . And thus it entails other approaches to mitigate the lack of exploration.

\section{Conclusion}
In this paper, we present an entropy augmentation method to mitigate the exploration issue of TRPO and PPO, in a general, simple and effective way. We first observed the inconsistency of the SAC, and argued that it implicitly optimizes another objective as proposed in this paper. Our method bears similar properties of trust region methods, and is compatible with the variance reduction mechanism GAE. Our analysis suggests that it would be crucial to anneal the temperature parameter. And our empirical results show that our method indeed encourages the exploration, and can aid the agent converge fast, escape from suboptimums and achieve better performance. Future work would investigate methods that auto-tuning the temperature coefficient, such as that done in \cite{DBLP:journals/corr/abs-1812-05905} \cite{DBLP:journals/corr/abs-2112-02852}.

% References
\bibliography{main.bib}

\begin{thebibliography}{27}
\providecommand{\natexlab}[1]{#1}
\providecommand{\url}[1]{\texttt{#1}}
\expandafter\ifx\csname urlstyle\endcsname\relax
  \providecommand{\doi}[1]{doi: #1}\else
  \providecommand{\doi}{doi: \begingroup \urlstyle{rm}\Url}\fi

\bibitem[Ahmed et~al.(2019)Ahmed, Roux, Norouzi, and
  Schuurmans]{DBLP:conf/icml/AhmedR0S19}
Zafarali Ahmed, Nicolas~Le Roux, Mohammad Norouzi, and Dale Schuurmans.
\newblock Understanding the impact of entropy on policy optimization.
\newblock In Kamalika Chaudhuri and Ruslan Salakhutdinov, editors,
  \emph{Proceedings of the 36th International Conference on Machine Learning,
  {ICML} 2019, 9-15 June 2019, Long Beach, California, {USA}}, volume~97 of
  \emph{Proceedings of Machine Learning Research}, pages 151--160. {PMLR},
  2019.
\newblock URL \url{http://proceedings.mlr.press/v97/ahmed19a.html}.

\bibitem[Brockman et~al.(2016)Brockman, Cheung, Pettersson, Schneider,
  Schulman, Tang, and Zaremba]{DBLP:journals/corr/BrockmanCPSSTZ16}
Greg Brockman, Vicki Cheung, Ludwig Pettersson, Jonas Schneider, John Schulman,
  Jie Tang, and Wojciech Zaremba.
\newblock Openai gym.
\newblock \emph{CoRR}, abs/1606.01540, 2016.
\newblock URL \url{http://arxiv.org/abs/1606.01540}.

\bibitem[Burda et~al.(2019)Burda, Edwards, Storkey, and
  Klimov]{DBLP:conf/iclr/BurdaESK19}
Yuri Burda, Harrison Edwards, Amos~J. Storkey, and Oleg Klimov.
\newblock Exploration by random network distillation.
\newblock In \emph{7th International Conference on Learning Representations,
  {ICLR} 2019, New Orleans, LA, USA, May 6-9, 2019}. OpenReview.net, 2019.
\newblock URL \url{https://openreview.net/forum?id=H1lJJnR5Ym}.

\bibitem[Burnetas and Katehakis(1997)]{DBLP:journals/mor/BurnetasK97}
Apostolos~N. Burnetas and Michael~N. Katehakis.
\newblock Optimal adaptive policies for markov decision processes.
\newblock \emph{Math. Oper. Res.}, 22\penalty0 (1):\penalty0 222--255, 1997.
\newblock \doi{10.1287/moor.22.1.222}.
\newblock URL \url{https://doi.org/10.1287/moor.22.1.222}.

\bibitem[Deisenroth et~al.(2013)Deisenroth, Neumann, and
  Peters]{DBLP:journals/ftrob/DeisenrothNP13}
Marc~Peter Deisenroth, Gerhard Neumann, and Jan Peters.
\newblock A survey on policy search for robotics.
\newblock \emph{Found. Trends Robotics}, 2\penalty0 (1-2):\penalty0 1--142,
  2013.
\newblock \doi{10.1561/2300000021}.
\newblock URL \url{https://doi.org/10.1561/2300000021}.

\bibitem[Flet{-}Berliac et~al.(2021)Flet{-}Berliac, Ferret, Pietquin, Preux,
  and Geist]{DBLP:conf/iclr/Flet-BerliacFPP21}
Yannis Flet{-}Berliac, Johan Ferret, Olivier Pietquin, Philippe Preux, and
  Matthieu Geist.
\newblock Adversarially guided actor-critic.
\newblock In \emph{9th International Conference on Learning Representations,
  {ICLR} 2021, Virtual Event, Austria, May 3-7, 2021}. OpenReview.net, 2021.
\newblock URL \url{https://openreview.net/forum?id=\_mQp5cr\_iNy}.

\bibitem[Haarnoja et~al.(2017)Haarnoja, Tang, Abbeel, and
  Levine]{DBLP:conf/icml/HaarnojaTAL17}
Tuomas Haarnoja, Haoran Tang, Pieter Abbeel, and Sergey Levine.
\newblock Reinforcement learning with deep energy-based policies.
\newblock In Doina Precup and Yee~Whye Teh, editors, \emph{Proceedings of the
  34th International Conference on Machine Learning, {ICML} 2017, Sydney, NSW,
  Australia, 6-11 August 2017}, volume~70 of \emph{Proceedings of Machine
  Learning Research}, pages 1352--1361. {PMLR}, 2017.
\newblock URL \url{http://proceedings.mlr.press/v70/haarnoja17a.html}.

\bibitem[Haarnoja et~al.(2018{\natexlab{a}})Haarnoja, Zhou, Abbeel, and
  Levine]{DBLP:conf/icml/HaarnojaZAL18}
Tuomas Haarnoja, Aurick Zhou, Pieter Abbeel, and Sergey Levine.
\newblock Soft actor-critic: Off-policy maximum entropy deep reinforcement
  learning with a stochastic actor.
\newblock In Jennifer~G. Dy and Andreas Krause, editors, \emph{Proceedings of
  the 35th International Conference on Machine Learning, {ICML} 2018,
  Stockholmsm{\"{a}}ssan, Stockholm, Sweden, July 10-15, 2018}, volume~80 of
  \emph{Proceedings of Machine Learning Research}, pages 1856--1865. {PMLR},
  2018{\natexlab{a}}.
\newblock URL \url{http://proceedings.mlr.press/v80/haarnoja18b.html}.

\bibitem[Haarnoja et~al.(2018{\natexlab{b}})Haarnoja, Zhou, Hartikainen,
  Tucker, Ha, Tan, Kumar, Zhu, Gupta, Abbeel, and
  Levine]{DBLP:journals/corr/abs-1812-05905}
Tuomas Haarnoja, Aurick Zhou, Kristian Hartikainen, George Tucker, Sehoon Ha,
  Jie Tan, Vikash Kumar, Henry Zhu, Abhishek Gupta, Pieter Abbeel, and Sergey
  Levine.
\newblock Soft actor-critic algorithms and applications.
\newblock \emph{CoRR}, abs/1812.05905, 2018{\natexlab{b}}.
\newblock URL \url{http://arxiv.org/abs/1812.05905}.

\bibitem[Han and Sung(2021)]{DBLP:conf/icml/HanS21}
Seungyul Han and Youngchul Sung.
\newblock Diversity actor-critic: Sample-aware entropy regularization for
  sample-efficient exploration.
\newblock In Marina Meila and Tong Zhang, editors, \emph{Proceedings of the
  38th International Conference on Machine Learning, {ICML} 2021, 18-24 July
  2021, Virtual Event}, volume 139 of \emph{Proceedings of Machine Learning
  Research}, pages 4018--4029. {PMLR}, 2021.
\newblock URL \url{http://proceedings.mlr.press/v139/han21a.html}.

\bibitem[Hsu et~al.(2020)Hsu, Mendler{-}D{\"{u}}nner, and
  Hardt]{DBLP:journals/corr/abs-2009-10897}
Chloe~Ching{-}Yun Hsu, Celestine Mendler{-}D{\"{u}}nner, and Moritz Hardt.
\newblock Revisiting design choices in proximal policy optimization.
\newblock \emph{CoRR}, abs/2009.10897, 2020.
\newblock URL \url{https://arxiv.org/abs/2009.10897}.

\bibitem[Kakade and Langford(2002)]{DBLP:conf/icml/KakadeL02}
Sham~M. Kakade and John Langford.
\newblock Approximately optimal approximate reinforcement learning.
\newblock In Claude Sammut and Achim~G. Hoffmann, editors, \emph{Machine
  Learning, Proceedings of the Nineteenth International Conference {(ICML}
  2002), University of New South Wales, Sydney, Australia, July 8-12, 2002},
  pages 267--274. Morgan Kaufmann, 2002.

\bibitem[Mnih et~al.(2015)Mnih, Kavukcuoglu, Silver, Rusu, Veness, Bellemare,
  Graves, Riedmiller, Fidjeland, Ostrovski, Petersen, Beattie, Sadik,
  Antonoglou, King, Kumaran, Wierstra, Legg, and
  Hassabis]{DBLP:journals/nature/MnihKSRVBGRFOPB15}
Volodymyr Mnih, Koray Kavukcuoglu, David Silver, Andrei~A. Rusu, Joel Veness,
  Marc~G. Bellemare, Alex Graves, Martin~A. Riedmiller, Andreas Fidjeland,
  Georg Ostrovski, Stig Petersen, Charles Beattie, Amir Sadik, Ioannis
  Antonoglou, Helen King, Dharshan Kumaran, Daan Wierstra, Shane Legg, and
  Demis Hassabis.
\newblock Human-level control through deep reinforcement learning.
\newblock \emph{Nat.}, 518\penalty0 (7540):\penalty0 529--533, 2015.
\newblock \doi{10.1038/nature14236}.
\newblock URL \url{https://doi.org/10.1038/nature14236}.

\bibitem[Mnih et~al.(2016)Mnih, Badia, Mirza, Graves, Lillicrap, Harley,
  Silver, and Kavukcuoglu]{DBLP:conf/icml/MnihBMGLHSK16}
Volodymyr Mnih, Adri{\`{a}}~Puigdom{\`{e}}nech Badia, Mehdi Mirza, Alex Graves,
  Timothy~P. Lillicrap, Tim Harley, David Silver, and Koray Kavukcuoglu.
\newblock Asynchronous methods for deep reinforcement learning.
\newblock In Maria{-}Florina Balcan and Kilian~Q. Weinberger, editors,
  \emph{Proceedings of the 33nd International Conference on Machine Learning,
  {ICML} 2016, New York City, NY, USA, June 19-24, 2016}, volume~48 of
  \emph{{JMLR} Workshop and Conference Proceedings}, pages 1928--1937.
  JMLR.org, 2016.
\newblock URL \url{http://proceedings.mlr.press/v48/mniha16.html}.

\bibitem[Nachum et~al.(2017)Nachum, Norouzi, Xu, and
  Schuurmans]{DBLP:conf/nips/NachumNXS17}
Ofir Nachum, Mohammad Norouzi, Kelvin Xu, and Dale Schuurmans.
\newblock Bridging the gap between value and policy based reinforcement
  learning.
\newblock In Isabelle Guyon, Ulrike von Luxburg, Samy Bengio, Hanna~M. Wallach,
  Rob Fergus, S.~V.~N. Vishwanathan, and Roman Garnett, editors, \emph{Advances
  in Neural Information Processing Systems 30: Annual Conference on Neural
  Information Processing Systems 2017, December 4-9, 2017, Long Beach, CA,
  {USA}}, pages 2775--2785, 2017.
\newblock URL
  \url{https://proceedings.neurips.cc/paper/2017/hash/facf9f743b083008a894eee7baa16469-Abstract.html}.

\bibitem[Neu et~al.(2017)Neu, Jonsson, and
  G{\'{o}}mez]{DBLP:journals/corr/NeuJG17}
Gergely Neu, Anders Jonsson, and Vicen{\c{c}} G{\'{o}}mez.
\newblock A unified view of entropy-regularized markov decision processes.
\newblock \emph{CoRR}, abs/1705.07798, 2017.
\newblock URL \url{http://arxiv.org/abs/1705.07798}.

\bibitem[Ng et~al.(1999)Ng, Harada, and Russell]{DBLP:conf/icml/NgHR99}
Andrew~Y. Ng, Daishi Harada, and Stuart Russell.
\newblock Policy invariance under reward transformations: Theory and
  application to reward shaping.
\newblock In Ivan Bratko and Saso Dzeroski, editors, \emph{Proceedings of the
  Sixteenth International Conference on Machine Learning {(ICML} 1999), Bled,
  Slovenia, June 27 - 30, 1999}, pages 278--287. Morgan Kaufmann, 1999.

\bibitem[Schulman et~al.(2015)Schulman, Levine, Abbeel, Jordan, and
  Moritz]{DBLP:conf/icml/SchulmanLAJM15}
John Schulman, Sergey Levine, Pieter Abbeel, Michael~I. Jordan, and Philipp
  Moritz.
\newblock Trust region policy optimization.
\newblock In Francis~R. Bach and David~M. Blei, editors, \emph{Proceedings of
  the 32nd International Conference on Machine Learning, {ICML} 2015, Lille,
  France, 6-11 July 2015}, volume~37 of \emph{{JMLR} Workshop and Conference
  Proceedings}, pages 1889--1897. JMLR.org, 2015.
\newblock URL \url{http://proceedings.mlr.press/v37/schulman15.html}.

\bibitem[Schulman et~al.(2016)Schulman, Moritz, Levine, Jordan, and
  Abbeel]{DBLP:journals/corr/SchulmanMLJA15}
John Schulman, Philipp Moritz, Sergey Levine, Michael~I. Jordan, and Pieter
  Abbeel.
\newblock High-dimensional continuous control using generalized advantage
  estimation.
\newblock In Yoshua Bengio and Yann LeCun, editors, \emph{4th International
  Conference on Learning Representations, {ICLR} 2016, San Juan, Puerto Rico,
  May 2-4, 2016, Conference Track Proceedings}, 2016.
\newblock URL \url{http://arxiv.org/abs/1506.02438}.

\bibitem[Schulman et~al.(2017{\natexlab{a}})Schulman, Abbeel, and
  Chen]{DBLP:journals/corr/SchulmanAC17}
John Schulman, Pieter Abbeel, and Xi~Chen.
\newblock Equivalence between policy gradients and soft q-learning.
\newblock \emph{CoRR}, abs/1704.06440, 2017{\natexlab{a}}.
\newblock URL \url{http://arxiv.org/abs/1704.06440}.

\bibitem[Schulman et~al.(2017{\natexlab{b}})Schulman, Wolski, Dhariwal,
  Radford, and Klimov]{DBLP:journals/corr/SchulmanWDRK17}
John Schulman, Filip Wolski, Prafulla Dhariwal, Alec Radford, and Oleg Klimov.
\newblock Proximal policy optimization algorithms.
\newblock \emph{CoRR}, abs/1707.06347, 2017{\natexlab{b}}.
\newblock URL \url{http://arxiv.org/abs/1707.06347}.

\bibitem[Sutton et~al.(1999)Sutton, McAllester, Singh, and
  Mansour]{DBLP:conf/nips/SuttonMSM99}
Richard~S. Sutton, David~A. McAllester, Satinder Singh, and Yishay Mansour.
\newblock Policy gradient methods for reinforcement learning with function
  approximation.
\newblock In Sara~A. Solla, Todd~K. Leen, and Klaus{-}Robert M{\"{u}}ller,
  editors, \emph{Advances in Neural Information Processing Systems 12, {[NIPS}
  Conference, Denver, Colorado, USA, November 29 - December 4, 1999]}, pages
  1057--1063. The {MIT} Press, 1999.
\newblock URL
  \url{http://papers.nips.cc/paper/1713-policy-gradient-methods-for-reinforcement-learning-with-function-approximation}.

\bibitem[Todorov et~al.(2012)Todorov, Erez, and
  Tassa]{DBLP:conf/iros/TodorovET12}
Emanuel Todorov, Tom Erez, and Yuval Tassa.
\newblock Mujoco: {A} physics engine for model-based control.
\newblock In \emph{2012 {IEEE/RSJ} International Conference on Intelligent
  Robots and Systems, {IROS} 2012, Vilamoura, Algarve, Portugal, October 7-12,
  2012}, pages 5026--5033. {IEEE}, 2012.
\newblock \doi{10.1109/IROS.2012.6386109}.
\newblock URL \url{https://doi.org/10.1109/IROS.2012.6386109}.

\bibitem[WILLIAMS and PENG(1991)]{doi:10.1080/09540099108946587}
RONALD~J. WILLIAMS and JING PENG.
\newblock Function optimization using connectionist reinforcement learning
  algorithms.
\newblock \emph{Connection Science}, 3\penalty0 (3):\penalty0 241--268, 1991.
\newblock \doi{10.1080/09540099108946587}.
\newblock URL \url{https://doi.org/10.1080/09540099108946587}.

\bibitem[Xu et~al.(2021)Xu, Hu, Liang, McAleer, Abbeel, and
  Fox]{DBLP:journals/corr/abs-2112-02852}
Yaosheng Xu, Dailin Hu, Litian Liang, Stephen McAleer, Pieter Abbeel, and Roy
  Fox.
\newblock Target entropy annealing for discrete soft actor-critic.
\newblock \emph{CoRR}, abs/2112.02852, 2021.
\newblock URL \url{https://arxiv.org/abs/2112.02852}.

\bibitem[Yang et~al.(2019)Yang, Li, and Zhang]{DBLP:conf/nips/YangLZ19}
Wenhao Yang, Xiang Li, and Zhihua Zhang.
\newblock A regularized approach to sparse optimal policy in reinforcement
  learning.
\newblock In Hanna~M. Wallach, Hugo Larochelle, Alina Beygelzimer, Florence
  d'Alch{\'{e}}{-}Buc, Emily~B. Fox, and Roman Garnett, editors, \emph{Advances
  in Neural Information Processing Systems 32: Annual Conference on Neural
  Information Processing Systems 2019, NeurIPS 2019, December 8-14, 2019,
  Vancouver, BC, Canada}, pages 5938--5948, 2019.
\newblock URL
  \url{https://proceedings.neurips.cc/paper/2019/hash/3f4366aeb9c157cf9a30c90693eafc55-Abstract.html}.

\bibitem[Ziebart(2010)]{DBLP:phd/us/Ziebart18}
Brian~D. Ziebart.
\newblock \emph{Modeling Purposeful Adaptive Behavior with the Principle of
  Maximum Causal Entropy}.
\newblock PhD thesis, Carnegie Mellon University, {USA}, 2010.
\newblock URL \url{https://doi.org/10.1184/r1/6720692.v1}.

\end{thebibliography}
\end{document}